%% file: main.tex
\documentclass{article}

\usepackage[square]{natbib}
\usepackage{todonotes}
\usepackage[preprint]{neurips_2019}

\usepackage[utf8]{inputenc} 
\usepackage[T1]{fontenc}    
\usepackage{hyperref}       
\usepackage{url}            
\usepackage{booktabs}       
\usepackage{amsfonts}       
\usepackage{nicefrac}       
\usepackage{microtype}      

\usepackage{times}
\usepackage{epsfig}
\usepackage{graphicx}
\usepackage{amsmath}
\usepackage{amssymb}

\usepackage{amsthm}  
\usepackage{wasysym}

\usepackage{color}
\usepackage{dsfont}	
\usepackage{wrapfig}
\usepackage{footnote}
\usepackage{epstopdf}
\usepackage{enumitem}
\usepackage{subfigure}
\usepackage{csquotes}
\usepackage{todonotes}

\def\eg{\emph{e.g. }}
\def\ie{\emph{i.e. }}
\def\cf{\emph{c.f. }}
 
\def\vs{\emph{vs. }}
\def\wrt{\emph{w.r.t. }}
\def\etal{\emph{et al. }}

\def \x {\mathbf{x}}
\def \h {\mathbf{h}}
\def \Q {\mathbf{Q}}

\makeatletter
\newcommand*{\rom}[1]{\expandafter\@slowromancap\romannumeral #1@}
\makeatother
\makeatletter
\newcommand\footnoteref[1]{\protected@xdef\@thefnmark{\ref{#1}}\@footnotemark}
\makeatother

\newtheorem{lemma}{Lemma}

\newtheorem{thm}{Theorem}
\newtheorem{cor}{Corollary}

\title{RNNs Evolving on an Equilibrium Manifold:\\ A Panacea for Vanishing and Exploding Gradients?}

\author{Anil Kag \\ 
Boston University \\ Boston, MA 02215\\
\texttt{anilkag@bu.edu}
	\And
	Ziming Zhang \\
	MERL\\
	Cambridge, MA 02139-1955 \\
	\texttt{zzhang@merl.com} \\
	 \And
	 Venkatesh Saligrama \\
	 Boston University \\ Boston, MA 02215\\
	 \texttt{srv@bu.edu} \\
}

\begin{document}
	
\maketitle
	
\begin{abstract}
\input{abstract_rev.tex}

%
%
%
%
\end{abstract}

\section{Introduction}\label{sec:intr}
\input{intro_rev2.tex}



\subsection{Related Work}
\input{related_rev.tex}

\section{Equilibriated Recurrent Neural Networks (ERNN)}\label{sec:ERNN}
\input{method_rev1.tex}

\section{Experiments}\label{sec:exp}
\input{expt_rev3.tex}

\section{Conclusion}
We developed new recurrent models (ERNNs) that evolve on the Equilibrium manifold of an ordinary differential equation. A key insight is to respond to a new signal input, by updating the state-vector that corresponds to a point on the manifold, rather than taking a direction pointed to by the vector-field. This seemingly simple idea has significant implications leading to better control of state vectors, since the partial derivatives are now evaluated with respect to the feasible manifold directions. As a consequence, we find that for a relatively modest modification of the ODE, these partial derivatives neither vanish nor decay, thus justifying our approach. The proposed solution leads us to novel recurrent rules and smaller model size architectures, which can be efficiently trained. ERNNs achieve state-of-the-art accuracy on many challenging data sets with 3-10x speedups and 1.5-3x model size reduction over the state-of-art, with prediction cost similar to vanilla RNNs.
\newpage
\bibliographystyle{ieee}
\bibliography{egbib}
	
\newpage
{\Huge \textbf{Supplementary}}

\input{supplementary.tex}

\end{document}

%% file: abstract_rev.tex
Recurrent neural networks (RNNs) are particularly well-suited for modeling long-term dependencies in sequential data, but are notoriously hard to train because the error backpropagated in time either vanishes or explodes at an exponential rate. While a number of works attempt to mitigate this effect through gated recurrent units, well-chosen parametric constraints, and skip-connections, we develop a novel perspective that seeks to evolve the hidden state on the equilibrium manifold of an ordinary differential equation (ODE). We propose a family of novel RNNs, namely {\em Equilibriated Recurrent Neural Networks} (ERNNs) that overcome the gradient decay or explosion effect and lead to recurrent models that evolve on the equilibrium manifold. We show that equilibrium points are stable, leading to fast convergence of the discretized ODE to fixed points. Furthermore, ERNNs account for long-term dependencies, and can efficiently recall informative aspects of data from the distant past. We show that ERNNs achieve state-of-the-art accuracy on many challenging data sets with 3-10x speedups, 1.5-3x model size reduction, and with similar prediction cost relative to vanilla RNNs.


%% file: intro_rev2.tex
Recurrent neural networks (RNNs) store a hidden state vector $\h_{k} \in \mathbb{R}^D$ that evolves over time, $k \in \{1,2,\ldots, T\} \triangleq [T]$ starting from an initial state $\h_0=\hat{\h}$. The hidden state vector is updated upon receiving a sequential input signal $x = \{\x_k\} \subseteq \mathbb{R}^d$ using a non-linear activation function, $\phi$, that is applied component-wise to a linearly transformed hidden state and input vector.
\begin{equation} \label{eq.vanilla}  \mathbf{h}_k=\phi(\mathbf{U}\mathbf{h}_{k-1}+\mathbf{W}\mathbf{x}_k+\mathbf{b}), 
\end{equation} 
This seemingly simple update rule has had significant success, accounting for long-term dependencies in sequential inputs and time-series data resulting in complex decision rules. 
Nevertheless, training RNNs is notoriously difficult, as shown in the work of \cite{Bengio:1994:LLD:2325857.2328340} and \cite{pascanu2013difficulty}: the gradient of loss backpropagated through time suffers from exponential decay or explosion. Closer inspection reveals that the partial derivative of the empirical loss function with respect to parameters $\mathbf{U},\mathbf{W}$ in turn requires computation of $\frac{\partial \mathbf{h}_m}{\partial \mathbf{h}_n}$ for indices $m>n$, which is a product of $m-n$ terms: 
\begin{align}\label{eqn:exp}
\frac{\partial \mathbf{h}_m}{\partial \mathbf{h}_n} = \prod_{m\geq k>n}\frac{\partial \mathbf{h}_k}{\partial \mathbf{h}_{k-1}} = \prod_{m\geq k>n} \nabla\phi(\mathbf{U}\mathbf{h}_{k-1} + \mathbf{W}\mathbf{x}_{k} + \mathbf{b})\mathbf{U}.
\end{align}
where $\nabla\phi\in\mathbb{R}^{D\times D}$ denotes a diagonal matrix based on the (sub)gradient of $\phi$, which in the RNN setting has non-negative components (\eg sigmoid, $\tanh$, or ReLU). Barring pathological cases, the matrix product induces either an exponentially exploding or exponentially decaying value leading to poor error back-propagation during training. 


A number of works have been proposed to combat the gradient explosion and decay. 
Within the recurrent network literature, gated architectures, constrained weight matrices, and residual connections have been proposed. Different from these works, we draw inspiration from a recent line of research ~\cite{nde,chang2018antisymmetricrnn,tallec2018can}, which leverages insights from discretizing ordinary differential equations (ODEs) to propose new recurrent models for mitigating gradient explosion and decay. One of these insights is to view residual connections as an instantiation of the forward Euler discretization method: 
\begin{equation} \label{eq.ode_resid}
{d\h(t)\over dt} \triangleq 
\mathbf{h}'(t)= \phi(\mathbf{U}\mathbf{h}(t)+\mathbf{W}\mathbf{x}_k+\mathbf{b}) \stackrel{\mbox{Euler}}{\Longrightarrow} \mathbf{h}_k= \mathbf{h}_{k-1} + \eta \phi(\mathbf{U}\mathbf{h}_{k-1}+\mathbf{W}\mathbf{x}_k+\mathbf{b}). 
\end{equation}
Observe that with such residual connections (\ie the right term in Eq.~\ref{eq.ode_resid}), we can have $\frac{\partial \mathbf{h}_{k}}{\partial \mathbf{h}_{k-1}} = \mathbf{I} + \eta \nabla\phi(\mathbf{U}\mathbf{h}_{k-1}+\mathbf{W}\mathbf{x}_k+\mathbf{b})\mathbf{U}$. 
A {\it key insight} of \cite{kusupati2018nips} is that the chain of matrix products is {\it reasonably} well behaved so long as $\eta$ is chosen to be sufficiently small. Nevertheless, requiring $\eta$ small dilutes the contribution of the instantaneous signal component, which is somewhat undesirable. 

{\bf Contributions.} Motivated by these works, we develop new insights centered around the idea of discrete recursions that converge to the equilibrium solution of an ODE, which is much like Eq.~\ref{eq.ode_resid}. {\it Note that we respond to a new signal input, by updating the state-vector that corresponds to a new equilibrium point, rather than taking a direction pointed to by the vector-field.} This seemingly simple idea has significant implications leading to better control of state vectors, since the partial derivatives are now evaluated with respect to the feasible manifold directions. As a consequence, we find that for a relatively modest modification of the ODE, these partial derivatives neither vanish nor decay, thus justifying our approach. The proposed solution leads us to novel recurrent rules and smaller model size architectures, which can be efficiently trained. We refer to these models as {\em Equilibriated Recurrent Neural Networks} (ERNNs). Our ERNNs achieve state-of-the-art accuracy on many challenging data sets with 3-10x speedups, 1.5-3x model size reduction over the state-of-art, and similar prediction cost to vanilla RNNs.

\begin{wrapfigure}{r}{.47\linewidth}
\vspace{-30pt}
	\begin{center}
		\includegraphics[width=\linewidth]{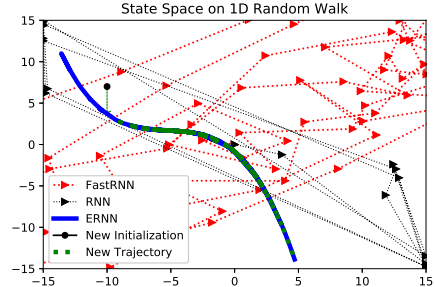}
		\vspace{-10mm}
		\caption{\footnotesize Phase-space trajectory of RNN, FastRNN, ERNN. X-axis denotes 1st dimension, and Y-axis 2nd dimension of 2D hidden state subject to random walk input with variance 10.}
		\label{fig:state_space_random_walk}
	\end{center}
	\vspace{-20pt}
\end{wrapfigure}
\textbf{Phase-Space Plot.} Fig.~\ref{fig:state_space_random_walk} depicts hidden state trajectory for ERNN, FastRNN, and RNN for a $1$D random walk input with variance 10 for $1000$ time steps, for a $2$D hidden space. We use $\tanh$ as the activation for all models and parameter matrices $\mathbf{U}, \mathbf{W}, \mathbf{b}$ are initialized randomly. For RNN, states are scaled since ERNN and FastRNN are not restricted to be in the cube. Evidently, RNN and FastRNN exhibit complex trajectories. Proposed ERNN first projects initial condition (black circle) onto the {\em equilibrium manifold} (blue) and moves within the depicted manifold (green).

\subsection{Overview of Proposed Solution.} \label{sec:overview}
{\bf Implicit Method.}
To build intuition, we revisit the ODE (left-term) in Eq.~\ref{eq.ode_resid}. An alternative approach to Euler discretization is the implicit method \cite{butcher2003numerical}, namely,
\begin{equation} \label{eq.ode_implicit}
\mathbf{h}(0)=\hat{\mathbf{h}}, \mathbf{h}'(t) = \phi(\mathbf{U}\mathbf{h}(t)+\mathbf{W}\mathbf{x}_k+\mathbf{b}) \stackrel{\mbox{Implicit}}{\Longrightarrow} \mathbf{h}_k= \mathbf{h}_{k-1} + \eta \phi(\mathbf{U}\mathbf{h}_{k}+\mathbf{W}\mathbf{x}_k+\mathbf{b}),\mathbf{h}_0=\hat{\mathbf{h}}.
\end{equation}
Notice that the main difference from before is that $\h_k$ appears on both sides of the discrete recursion. The state update is no longer computationally straightforward since it requires solving an implicit equation. While this is a drawback, we leverage existing fixed point recursion methods in this context. A known benefit of implicit methods is that they are more stable than their explicit counterparts~\cite{butcher2003numerical}.

{\bf Fixed Point Recursions.}
To solve the implicit equation in Eq. \ref{eq.ode_implicit}, we consider a standard fixed point recursion based again on Euler's method, which in essence is a root-finding technique. We find that our fixed point recursion converges rapidly (fewer than 5 steps in our experiments) and in theory, convergence follows from standard results for non-stiff systems~\cite{butcher2003numerical}. With $\mathbf{h}_k^{(0)}=\mathbf{h}_{k-1}$, we follow 
\begin{equation} \label{eq.fpr}
\mathbf{h}_k^{(i)} = \mathbf{h}_{k-1} + \eta \phi(\mathbf{U}\mathbf{h}_{k}^{(i-1)}+\mathbf{W}\mathbf{x}_k+\mathbf{b}) \stackrel{i \rightarrow \infty}{\implies} \mathbf{h}_k^{*} = \mathbf{h}_{k-1} + \eta \phi(\mathbf{U}\mathbf{h}_{k}^{*}+\mathbf{W}\mathbf{x}_k+\mathbf{b}).
\end{equation}
A key observation here is that the implicit method results in a new update rule, $\kappa:\mathbb{R}^D \times \mathbb{R}^d \rightarrow \mathbb{R}^D$, taking a tuple $(\mathbf{h},\mathbf{x})\in \mathbb{R}^D \times \mathbb{R}^d$ and mapping it to a point $\mathbf{h}^*\in \mathbb{R}^D$, which is a fixed point of the recursion, and this fixed point is an equilibrium solution of the ODE. So its characteristics are no longer governed by Eq.~\ref{eq.vanilla} or Eq.~\ref{eq.ode_resid}. 

{\bf Partial Derivatives on the Equilibrium Manifold.}
Since the hidden state lies on the equilibrium set, we are motivated to examine its properties. We have $\mathbf{h}'(t)=\phi(\mathbf{U}\mathbf{h}(t)+\mathbf{W}\mathbf{x}_k +\mathbf{b}),\,\, \mathbf{h}(0)=\mathbf{h}_{k-1}$. Since the initial condition is itself a parameter, we reparameterize the ODE and rewrite it as $\tilde{ \mathbf{h} }(t) = \mathbf{h}(t)-\mathbf{h}_{k-1}$. With this choice we have $\tilde{ \mathbf{h}}'(t)=\phi(\mathbf{U}(\tilde{ \mathbf{h}}(t)+\mathbf{h}_{k-1})+\mathbf{W}\mathbf{x}_k +\mathbf{b}),\,\, \tilde{ \mathbf{h}}(0)=\mathbf{0}$. The equilibrium solutions are defined by the set:
$$
{\cal M} = \{(\mathbf{h}_{k-1},\mathbf{x}_k,\mathbf{h}^*)\in \mathbb{R}^D \times \mathbb{R}^d \times \mathbb{R}^D \mid \phi(\mathbf{U}(\mathbf{h}^*+\mathbf{h}_{k-1})+\mathbf{W}\mathbf{x}_k +\mathbf{b})=\mathbf{0}\}.
$$
As such solution constrains at most $D$ variables and we could equivalently view the manifold as a collection of equilibrium points, $\mathbf{h}^*$, over all admissible tuples $((\mathbf{h}_{k-1},\mathbf{x}_k)$. As such, we do not enforce constraints on admissible input sequences. Under general smoothness conditions on the function $\phi$, and that $\nabla \phi$ does not vanish at any point in its domain (which follows for sigmoid and $\tanh$ activations), ${\cal M}$ defines a differentiable manifold, a direct consequence of the implicit function theorem~\cite{spivak1970comprehensive}. Furthermore, as a result\footnote{\url{http://cosweb1.fau.edu/~jmirelesjames/ODE_course/lectureNotes_shortVersion_day1.pdf}}, we can again invoke the implicit function theorem, leveraging the fact that the equilibrium solution varies smoothly as a function of the parameters $\mathbf{h}_{k-1}$ and $\mathbf{x}_k$. This allows us to express the partial derivatives of $\tilde{ \mathbf{h}}$ with respect to $\mathbf{h}_{k-1}$. It follows that $\nabla \phi(\mathbf{U}(\tilde{ \mathbf{h}}(t)+\mathbf{h}_{k-1})+\mathbf{W}\mathbf{x}_k +\mathbf{b})\mathbf{U}(\frac{\partial \tilde{ \mathbf{h}}}{\partial \mathbf{h}_{k-1}}+\mathbf{I})=\mathbf{0}$. Consequently, unless $\mathbf{U}$ is singular (we already assumed $\nabla \phi$ does not vanish on ${\cal M}$), we get $\frac{\partial\tilde {\mathbf{h}}}{\partial\mathbf{h}_{k-1}}= -\mathbf{I}$. While this suggests that we have circumvented the gradient problem, but alas, our solution is for $\tilde{ \mathbf{h}}(t) = \mathbf{h}(t)-\mathbf{h}_{k-1}$, and so translating back to $\h$ we see that $\frac{\partial\mathbf{h}}{\partial\mathbf{h}_{k-1}}=\mathbf{0}$. 

{\bf Circumventing Gradient Explosion and Decay: Norm Preserving Activation.}
Nevertheless, what is interesting in the above argument is that the trajectory that evolves along the equilibrium points in continuous time or equivalently the fixed points of the corresponding recursion, has a markedly different behavior from vanilla RNNs or heretofore considered variants. Indeed, it follows that by suitably modifying our activation function, we can guarantee a unitary transformation of the desired partial derivatives. Specifically, consider the ODE: $$\mathbf{h}'(t)=\phi(\mathbf{U}(\mathbf{h}(t)+\mathbf{h}_{k-1})+\mathbf{W}\mathbf{x}_k+\mathbf{b})-(\mathbf{h}(t)+\mathbf{h}_{k-1}),\,\,\mathbf{h}(0)=\mathbf{0}.$$
On the corresponding equilibrium manifold (assuming we again satisfy the required conditions) we have  $\frac{\partial\mathbf{h}}{\partial\mathbf{h}_{k-1}}=-\mathbf{I}$, for non-singular $\mathbf{U}$ and $\nabla \phi(\cdot)$. It follows that the fixed point recursion, such as Eq.~\ref{eq.fpr}, associated with this variant has desirable properties.  

{\bf Summary.}
Based on this intuition, we propose an efficient implementation for fixed point recursion. Our state updates converge at a linear rate to the fixed point (and hence to the equilibrium solution of the ODE). In theory, we observe that for suitable choice of parameters, equilibrium points are locally stable resulting in rapid convergence. In addition, whenever the input is slowly varying, the corresponding state updates are smooth, and as such lie on a low dimensional manifold. Furthermore, ERNNs account for long-term dependencies, and can efficiently recall informative aspects of data from the distant past. Our framework is general, and applies for arbitrary time steps, as well as for deep multi-layer blocks. Indeed, we could replace the activation function $\phi(\cdot)$ for an arbitrary transition function.  

%% file: related_rev.tex
\textbf{Gated Architectures.}
Long short-term memory (LSTM) \cite{hochreiter1997long} is widely used in RNNs to model long-term dependency in sequential data. Gated recurrent unit (GRU) \cite{cho2014properties} is another gating mechanism that has been demonstrated to achieve similar performance of LSTM with fewer parameters. Some recent gated architectures include  UGRNN \cite{2016arXivUGRNN}, and FastGRNN \cite{kusupati2018nips}. Although these models try to mitigate the vanishing/exploding gradients, they do not circumvent it. Moreover these models usually lead to the increase in training cost, inference cost and model size.

\textbf{Unitary RNNs.} 
This is another family of RNNs \cite{arjovsky2016unitary, jing2017tunable, 2018SpectralRNN, 2017oRNN} that consist of well-conditioned state transition matrices. They achieve this via bookkeeping the transition matrix to a nearly unitary matrix, which limits their expressive power and prediction accuracy while
increasing training time. 

\textbf{Deep State Transition RNNs.} Deep transition and multiscale RNNs are proposed to introduce nonlinear transition functions into RNNs for performance improvement. For instance, \cite{pascanu2013construct} empirically analyzed the problem of how to construct deep RNNs. \cite{zilly2017recurrent} proposed extending the LSTM architecture to allow step-to-step transition depths larger than one. \cite{mujika2017fast} proposed incorporating the strengths of both multiscale RNNs and deep transition RNNs to learn complex transition functions from one timestep to the next. These methods add complexity to the model via additional parameters, such extensions are readily possible for ERNNs as well.

\textbf{Feedforward Nets \& Residual Connections.} 
Apart from the well-known issues of vanishing or exploding gradients, several authors~\cite{Oord2016WaveNetAG,Gehring2017ConvolutionalST,Vaswani2017AttentionIA,Dauphin2017LanguageMW,miller2018recurrent} have observed that sequential processing leads to large training and inference costs because RNNs are inherently not parallelizable. To overcome these shortcomings they have proposed methods that replace recurrent models with parallelizable feedforward models which, while parallelizable, somewhat truncates the receptive field. Skip or residual connections ~\cite{Jaeger07,Bengio2013AdvancesIO,chang2017dilated,campos2017skip,KusupatiSBKJV18} feed-forward state vectors, and have been proposed as a middle ground between feed-forward and recurrent models.

\textbf{ODE/Dynamical Perspective.} 
We focus on methods that are inspired by ODEs, since they are somewhat related to our method. Nevertheless, at a high-level, we are the first to propose fixed point recursions drawing upon evolution over the equilibrium manifold. Like us, there are others that have sought to address training stability of RNNs from the perspective of dynamical systems. \cite{talathi2015improving} proposed a modified weight initialization strategy based on a simple dynamical system perspective on weight initialization process that leads to successfully training RNNs composed of ReLUs. \cite{niu2019recurrent} analyzed RNN architectures using numerical methods of ODE and propose a family of ODE-RNNs. 

Our work is closely related to Chang \etal \cite{chang2018antisymmetricrnn}, who proposed Antisymmetric-RNN based on expressing the $\mathbf{U}$ matrix in Eq.~\ref{eq.ode_resid} as a difference of two matrices that are transposes of each other, \ie $\mathbf{U}=\mathbf{V}-\mathbf{V}^T$. In continuous time, the spectrum of $\mathbf{U}$ defined in this way has primarly imaginary eigenvalues. Nevertheless, as they point out the corresponding forward Euler discretization is unstable. Their solution is a damping factor to stabilize the recursion. In particular, they proposed to damp the dynamics by a damping factor $\gamma>0$:  
$\mathbf{h}_k=\mathbf{h}_{k-1} + \eta \phi((\mathbf{V}-\mathbf{V}^T-\gamma \mathbf{I})\mathbf{h}_{k-1}+\mathbf{W}\mathbf{x}_k+\mathbf{b}).$  
The impact of this proposal is quite similar to our discussion on residual connections in the context of Eq.~\ref{eq.ode_resid}. It follows that, $\frac{\partial \mathbf{h}_{k}}{\partial \mathbf{h}_{k-1}} = \mathbf{I} + \eta \nabla\phi(\cdot) (\mathbf{V}-\mathbf{V}^T-\gamma \mathbf{I})^T$, and the eigenvalues of the partial derivative are strictly smaller than one, leading to mitigating gradient decay but not circumventing it.

%% file: method_rev1.tex
\label{ssec:problem_definition}

Following the motivation in Sec.~\ref{sec:overview}, we present a general ODE model to expose the key property required for circumventing gradient decay/explosion. Our eventual choice is conventional.   
%
\begin{align}\label{eqn:ode}
\mathbf{h}'_k(t) = F(\mathbf{h}_k(t)) \stackrel{def}{=} f\left(\mathbf{h}_k(t) + \mathbf{h}_{k-1}, \mathbf{x}_k; \alpha\right) - \gamma\left(\mathbf{h}_k(t) + \mathbf{h}_{k-1}\right), \,\,\, \h_k(0)=\mathbf{0}.
\end{align}
where $\mathbf{h}_k(t), \mathbf{h}_{k-1}\in\mathbb{R}^D, \forall t$, $\gamma \in \mathbb{R}$ is some multiplier, and $f:\mathbb{R}^D\times\mathbb{R}^d\times\mathcal{A}\rightarrow\mathbb{R}^{D}$ denotes a proper continuous and differentiable function parameterized by $\alpha\in\mathcal{A}$. An important difference between prior works and this is that we now have two indices $k$, which we will eventually associate with discrete time, and $t$ that denotes continuous time dynamics. Like we described in the introduction, a straightforward way to think about this is to assume that we start with the previous state $\h_{k-1}$ and we obtain a new input $\x_k$. We then run the ODE and obtain the solution $\h_k(t)$ until convergence, hopefully to equilibrium. We then update the state vector, i.e., $\h_k \triangleq \h_k(\infty)$.  
We learn weight parameters of the ODE by solving an empirical risk minimization problem, that constrains the state to lie on the equilibrium manifold of the ODE, while starting with $\h_0=\hat{\mathbf{h}}$:
\begin{align}\label{eqn:ernn}
\min_{\gamma,\alpha, \omega}  \sum_{(x,y)\sim\mathcal{X}\times\mathcal{Y}}\ell(\mathbf{h}_{T}, y; \omega), \; \mbox{s.t.} \; f(\mathbf{h}_k + \mathbf{h}_{k-1}, \mathbf{x}_k; \alpha) - \gamma(\h_{k}+\h_{k-1})=\mathbf{0},\,\, k\in[T].
\end{align}
Note that constraints are implicit, and we defer discussion of a solution until later. 
%
To further build intuition, we present a few properties of our framework.
For future convenience let us define the equilibrium manifold. These are points $\h_{eq}$ that satisfy the equilibrium condition, for fixed $(\h_{k-1},\x_k)$
\begin{equation} \label{eq.projeq}
\h_{eq} \in {\cal M}(\h_{k-1},\x_k)=\{\h \in \mathbb{R}^D \mid f\left(\h + \h_{k-1}, \x_k; \alpha\right) - \gamma \left(\h + \h_{k-1}\right)=\mathbf{0}\}.    
\end{equation}
\begin{lemma}[Identity Transition Mapping]\label{lem:IdTM}
	Suppose $\nabla f(\mathbf{h}_{eq} + \mathbf{h}_{k-1}, \x_k; \alpha) - \gamma \mathbf{I}$ is non-singular, it follows that,
	$   \frac{\partial\mathbf{h}_{eq}}{\partial\mathbf{h}_{k-1}} = -\mathbf{I}.$
	where $\mathbf{I}\in\mathbb{R}^{D\times D}$ denotes the identity matrix. 
\end{lemma}
\begin{proof}
	Note that by our assumptions of smoothness and non-singularity of the Jacobian of $F$, equilibrium solutions define a differentiable manifold much in the same way we discussed in Sec.~\ref{sec:intr}. By taking the partial derivatives \wrt $\mathbf{h}_{k-1}$ in Eq. \ref{eqn:ode}, at the equilibrium points we have $[\nabla f(\mathbf{h}_{eq} + \mathbf{h}_{k-1}; \alpha) - \gamma \mathbf{I} ][\frac{\partial\mathbf{h}_{eq}}{\partial\mathbf{h}_{k-1}} + \mathbf{I}] = \mathbf{0}$. The proof follows from our assumptions.
\end{proof}
We notice that replacing $\mathbf{h}_{k-1}$ with $-\mathbf{h}_{k-1}$ in Eq. \ref{eqn:ode} will lead to $\frac{\partial\mathbf{h}_{eq}}{\partial\mathbf{h}_{k-1}} = \mathbf{I}$, which also has no impact on magnitudes of gradients. As a result, both choices are suitable for circumventing vanishing or exploding gradients during training, but still may converge to different local minima and thus result in different test-time performance. Furthermore, notice that the norm preserving property is somewhat insensitive to choices of $\gamma$, so long as the non-singular condition is satisfied. 

\begin{thm}[Identity Transition Mapping in Backpropagation of ERNN]\label{thm:1}
	Suppose that the assumptions in Lemma \ref{lem:IdTM} holds. Then for our ERNNs in Eq. \ref{eqn:ernn} we have
	\begin{align}\label{eqn:identity_bp}
	\frac{\partial \mathbf{h}_m}{\partial \mathbf{h}_n} = \prod_{m\geq k>n}\frac{\partial \mathbf{h}_k}{\partial \mathbf{h}_{k-1}} = (-1)^{m-n}\mathbf{I} \Rightarrow \left\|\frac{\partial \mathbf{h}_m}{\partial \mathbf{h}_n}\right\| = 1.
	\end{align}
\end{thm}
\begin{proof}
	By invoking Lemma~\ref{lem:IdTM} into Eq. \ref{eqn:identity_bp}, the proof follows.
\end{proof}

\begin{wrapfigure}{r}{.5\linewidth}
	\vspace{-20pt}
	\begin{center}
		\includegraphics[width=\linewidth]{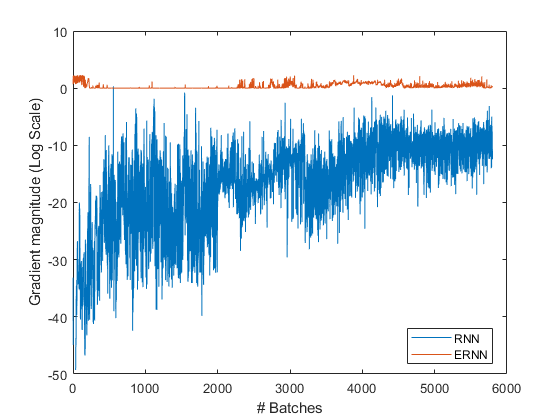}
		\vspace{-5mm}
		\caption{\footnotesize Comparison between RNN and ERNN on the magnitudes of gradients.}
		\label{fig:norm_grad}
	\end{center}
	\vspace{-15pt}
\end{wrapfigure}
To verify Theorem. \ref{thm:1} empirically, we train RNN and ERNN on the HAR-2 data set (see more details in Sec. \ref{sec:exp}), respectively, and plot in Fig.~\ref{fig:norm_grad} the magnitude of gradient of the last layer $\h_T$ \wrt the first layer $\h_1$ in log scale to confirm that our approach leads to no vanishing or exploding gradients when the error is back-propagated through time. 
We also conducted experiments to verify that the gradient of ERNN is norm preserving (see supplementary).
As we see clearly, RNN suffers from serious vanishing gradient issue in training, while ERNN's backpropagated gradients is close to 1, and the variance arises mainly our approximation of fixed points and stochastic behavior in training networks, 
demonstrating much better training stability of ERNN.

\noindent {\bf Multi-Layer RNN Blocks.} We point out in passing that our framework readily admits deep multi-layered networks within a single time-step. Indeed our setup is quite completely general; it applies to shallow and deep nets; small and large time steps. As a case in point, the Deep Transition RNN~\cite{deeptransitionRNN}: $$h_{k+1}=f_h(h_k,x_{k+1}) = \phi_h(W_L\phi_{L-1}(W_{L-1}\ldots W_1\phi_1(Uh_k+Wx_{k+1}))$$ is readily accounted by Theorem~1 in an implicit form: $$h_{k+1}=f_h(h_{k+1}+h_k,x_{k+1})-h_k.$$ So is Deep-RNN~\cite{trainingDeepRNN}. The trick is to transform $h_k \rightarrow h_k + h_{k+1}$ and $h_{k+1} \rightarrow h_{k}+h_{k+1}$. As such, all we need is smoothness of $f_h$, which has no restriction on \# layers. On the other hand, that we do not have to limit the number of time steps is the {\it point} of Theorem~1, which asserts that the partial differential of hidden states (which is primarily why vanishing/exploding gradient arises \cite{pascanu2013difficulty} in the first place) is identity!!

{\bf Uniform Asymptotic Stability on the Equilibrium Manifold.}
Next we point out that the equilibrium points, $h_{eq}$ in Eq.~\ref{eq.projeq}, are asymptotically stable for suitable choice of the activation function. 
\begin{lemma} \label{lem.stable}
Consider the dynamical system in Eq.~\ref{eqn:ode}. Suppose $\h_{eq}\in {\cal M}(\h_{k-1},\x_k)$, and the Jacobian of $F$ at $\h_{eq}$ has negative eigenvalues, then $h_{eq}$ is a locally asymptotically stable point (see \cite{vidyasagar}). 
\end{lemma}
To see that this property is generically satisfied, let $f(\h_k(t)+\h_{k-1},\mathbf{x}_k;\alpha)= \phi(\mathbf{U}(\h_k(t)+\h_{k-1})+\mathbf{W}\x_k+\mathbf{b})$ for some parameters $\mathbf{U} \in \mathbb{R}^{D\times D}, \mathbf{W} \in \mathbb{R}^{D\times d}$. It follows that we need matrix $\mathbf{M}=\nabla \phi(\cdot) \mathbf{U} - \gamma \mathbf{I}$ to have negative eigenvalues. Note that for conventional activation functions $\nabla \phi(\cdot)$ is a positive diagonal matrix with components bounded by one. Consequently, this condition suggests that we choose, $\gamma>0$ so that $\|\mathbf{U}\| \leq \gamma$. In our practical experiments, we did not attempt to learn $\gamma$ value but simply set it to one. 

{\bf Variation of Equilibrium \wrt Input.}
So far we have looked at various properties of equilibrium points \wrt the prior state. We consider its variation \wrt the input next. Again, let $f(\h_k(t)+\h_{k-1}, \x_k; \alpha)= \phi(\mathbf{U}(\h_k(t)+\h_{k-1})+\mathbf{W}\x_k+\mathbf{b})$ as above. Under the conditions of Lemma~\ref{lem:IdTM} and stability conditions in Lemma~\ref{lem.stable}, we are again in a position to locally express variations \wrt the input $\x_k$ at equilibrium. Let as before, $\h_{eq}$ be an equilibrium solution for some tuple $(\h_{k-1},\x_k)$. It follows that,
$$
(\gamma \mathbf{I}-\nabla \phi(\mathbf{U}(\h_k(t)+\h_{k-1})+\mathbf{W}\x_k+\mathbf{b})\mathbf{U})\partial \h_{eq} = \nabla \phi(\mathbf{U}(\h_k(t)+\h_{k-1})+\mathbf{W}\x_k+\mathbf{b}) \mathbf{W} \partial \x 
$$
This suggests that, whenever the input undergoes a slow variation, we expect that the equilibrium point moves in such a way that $\mathbf{U} \partial \h_{eq}$ must lie in a transformed span of $\mathbf{W}$. Now $\mathbf{W} \in \mathbb{R}^{D \times d}$ with $d \ll D$, which implies that $(\gamma \mathbf{I} -\nabla \phi(\mathbf{U}(\h_k(t)+\h_{k-1})+\mathbf{W}\x_k+\mathbf{b})\mathbf{U})$ is rank-deficient. 

{\bf Experimental Setup: Low Rank Matrix Parameterization}
For typical activation functions, note that whenever the argument is in the unsaturated regime $\nabla \phi(\cdot) \approx \mathbf{I}$. With this substitution, we approximately get $\mathrm{span}(\gamma \mathbf{I} - \mathbf{U}) \approx \mathrm{span}(\mathbf{W})$. We can express these constraints as $\mathbf{U}=\mathbf{I}+\mathbf{V}\mathbf{H}$ with low-rank matrices $\mathbf{V} \in \mathbb{R}^{D \times d_1}, \mathbf{H} \in \mathbb{R}^{d_1 \times D}$, and further project both $\mathbf{U}\h_k$ and $\mathbf{W}\x_k$ onto a shared space. Since in our experiments the signal vectors we encounter are low-dimensional, and sequential inputs vary slowly over time, we enforce this restriction in all our experiments. In particular, we consider,
%
%
\begin{equation} \label{eq.experiments}
\phi\left(\mathbf{P}[\mathbf{U}(\h_k(t)+\h_{k-1})+\mathbf{W}\x_k+\mathbf{b}]\right) - (\h_k(t) + \h_{k-1})=\mathbf{0}.
\end{equation}
The parameter matrix $\mathbf{P} \in \mathbb{R}^{D \times D}$ projects the contributions from input and hidden states onto the same space. To decrease model-size we let $\mathbf{P} =\mathbf{U}=(\mathbf{I}+\mathbf{V}\mathbf{H})$ learn these parameters. 

\subsection{The Euler Method for Solving fixed Points}
We consider the following standard update rule for solving fixed points \cite{butcher2003numerical}: 
\begin{align}\label{eqn:update_rule}
    \hspace{-2mm}\mathbf{h}_{k}^{(i+1)} & = \mathbf{h}_{k}^{(i)} + \eta_k^{(i)}F(\mathbf{h}_k^{(i)}) = \mathbf{h}_{k}^{(i)} + \eta_k^{(i)}[f(\mathbf{h}_k^{(i)} + \mathbf{h}_{k-1}, \mathbf{x}_k; \alpha) - (\mathbf{h}_k^{(i)} + \mathbf{h}_{k-1})],
\end{align}
where $\eta_k^{(i)}\in\mathbb{R}, \forall k, \forall i\in[K]$ denotes a small learnable constant. 

\begin{wrapfigure}{r}{.5\linewidth}
	\vspace{-5pt}
	\begin{center}
		\includegraphics[width=\linewidth]{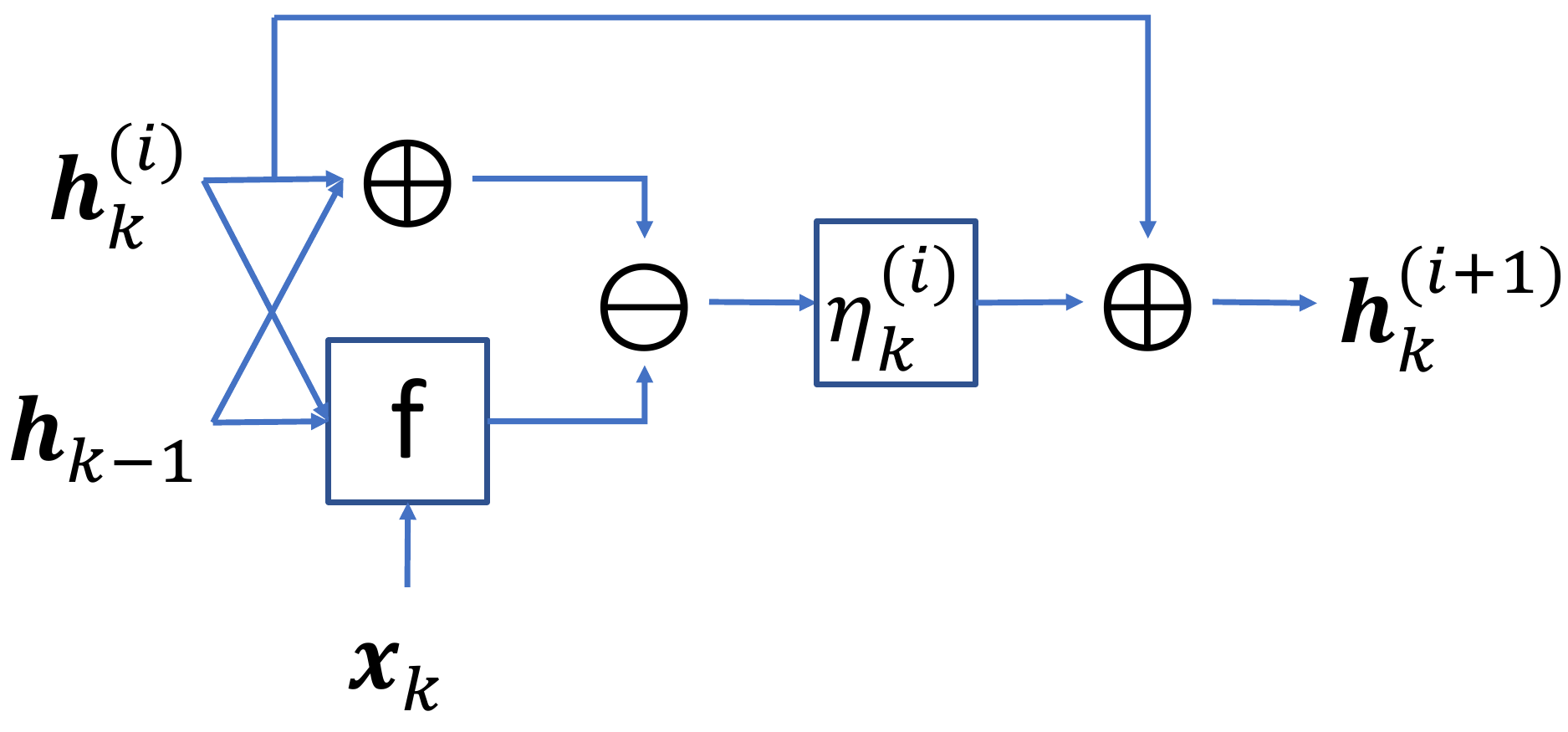}
		\vspace{-5mm}
		\caption{\footnotesize Illustration of network blocks for computing fixed points in ERNNs.}
		\label{fig:ERNN}
	\end{center}
	\vspace{-15pt}
\end{wrapfigure}
\textbf{Implementation:}
Such update rules can be efficiently implemented using networks, as illustrated in Fig. \ref{fig:ERNN}, where $\oplus, \ominus$ denote the entry-wise plus and minus operators, and each square denotes a learnable parameter or function with parameters of the network. Once the maximum number of iterations $K$ for solving fixed points is predefined, we can convert an ERNN into a feed-forward network using the blocks in Fig. \ref{fig:ERNN} to learn the RNN parameters simultaneously.  

\begin{thm}[Local Convergence with Linear Rate]\label{thm:convergence}
Assume that the function $F$ in Eq. \ref{eqn:ode} and the parameter $\eta_k^{(i)}$ in Eq. \ref{eqn:update_rule} satisfies 
\begin{align}\label{eqn:euler}
    [\eta_k^{(i)}]^2 \|\nabla F(\mathbf{h}_k^{(i)})F(\mathbf{h}_k^{(i)})\|^2 + 2\eta_k^{(i)}F(\mathbf{h}_k^{(i)})\nabla F(\mathbf{h}_k^{(i)})F(\mathbf{h}_k^{(i)}) < 0, \forall k, \forall i.
\end{align}
Then there exists $\epsilon>0$ such that if $\|\mathbf{h}_k^{(0)} - \mathbf{h}_k\|\leq\epsilon$ where $\mathbf{h}_k$ denotes the fixed point, the sequence $\{\mathbf{h}_k^{(i)}\}$ generated by the Euler method converges to the equilibrium solution in ${\cal M}(\h_{k-1},\mathbf{x}_k)$ locally with linear rate.
\end{thm}

The proof is based on drawing a connection between the Euler method and inexact Newton methods, and leverages Thm. 2.3 in \cite{dembo1982inexact}. See our supplementary (for proof, empirical verification).  

\begin{cor}
If $\|\mathbf{I}+\eta_k^{(i)}\nabla F(\mathbf{h}_k^{(i)})\|<1, \forall k, \forall i$, the forward propagation (Eq. \ref{eqn:update_rule}) is stable and the sequence $\{\mathbf{h}_k^{(i)}\}$ converges locally at a linear rate.
\end{cor}
The proof is based on Thm. 2.3 in \cite{dembo1982inexact}, Thm. \ref{thm:convergence} and Prop. 2 in \cite{chang2018antisymmetricrnn}. Please refer to our supplementary.

%% file: expt_rev3.tex
%

\begin{wraptable}{r}{.5\linewidth}
 	\vspace{-40pt}
	\begin{minipage}[b]{\linewidth}
	    \caption{\footnotesize Dataset Statistics}
		\begin{center}\footnotesize
            \setlength{\tabcolsep}{1pt}
            \begin{tabular}{|ccccc|} 
                \hline
                Data set & \#Train & \#Fts & \#Steps & \#Test \\ 
                \hline\hline
                DSA-19 & 4,560 & 5,625 & 125 & 4,560 \\
                HAR-2 & 7,352 & 1,152 & 128 & 2,947 \\
                Google-12 & 22,246 & 3,168 & 99 & 3,081 \\ 
                Google-30 & 51,088 & 3,168 & 99 & 6,835 \\
                PTB-10000 & 929,589 & - & 300 & 82,430 \\
                Yelp-5 & 500,000 & 38,400 & 300 & 500,000 \\
                Pixel-MNIST-10 & 60,000 & 784 & 784 & 10,000 \\
                Permuted-MNIST & 60,000 & 784 & 784 & 10,000 \\
                Noisy-MNIST & 60,000 & 784 & 1000 & 10,000 \\
                Noisy-CIFAR & 60,000 & 3072 & 1000  & 10,000 \\
                \hline
            \end{tabular}
            \vspace{-1mm}
            \label{table:1}
		\end{center}
	\end{minipage}
\end{wraptable}

\textbf{Setup:}
Our experiments are geared towards showing the importance of vanishing and exploding gradients. In this light, we'll consider variety of datasets exhibiting long term dependence (LTD). LTD datasets highlight the importance of recalling information over a very large time window. Hence, if the method suffers from ill-defined gradients, it'll not be able to recover the required information. \cite{chang2018antisymmetricrnn} argues that a method's LTD response is best seen on those datasets where only a smaller unknown segment(s) of a longer sequence is informative. This precisely occurs in our datasets (table \ref{table:ltd-nums}); as seen activity is a short time-span within a longer sequence; start time is random, so method has to be agnostic. 
 In order to perform a meaningful comparative study, our experimental setup will conform to state-of-art methods that satisfy the following conditions:\\ (a) methods that are devoid of additional application or dataset-specific heuristics; \\(b) methods that leverage only single cell/block/layer, and \\ (c) methods that do not benefit from complementary add-ons (such as \cite{chang2018antisymmetricrnn}+gating, \cite{kusupati2018nips}+gating, advanced regularization etc.). 
 
 Requiring (a) is not controversial since our goal is methodological here. Conditions (b),(c) are justifiable since we could also leverage these add-ons and are not germane to any particular method. We will demonstrate that on LTD tasks, ERNN outperforms comparable methods when these experimental conditions are enforced. . 


\begin{wraptable}{r}{.5\linewidth}
	\vspace{-20pt}
	\begin{minipage}[b]{\linewidth}
	    \caption{\footnotesize Long Term Dependence in Datasets}
		\begin{center}\footnotesize
            \setlength{\tabcolsep}{.5pt}
            \begin{tabular}{|c c c c|} 
                \hline
                Dataset & \begin{tabular}[c]{@{}c@{}}Avg. Acitivity\\Time\end{tabular} & \begin{tabular}[c]{@{}c@{}}Input \\Time \end{tabular} & \begin{tabular}[c]{@{}c@{}}Sequence \\ Ratio\end{tabular} \\ [0.5ex] 
                \hline\hline
                Google-12; 30 & 25ms & 1000ms & 3/99 \\
                DSA-19 & 500ms & 5000ms & 13/125 \\
                HAR-2 & 256ms & 2560ms & 13/128 \\
                Yelp-5 & 20 & 300 & 1/15 \\
                \hline
            \end{tabular}
            \label{table:ltd-nums}
		\end{center}
	\end{minipage}
\vspace{-13pt}
\end{wraptable}


\textbf{Datasets:} 
ERNN's performance was benchmarked against datasets that other works have shown to require LTD. The most common ones for LTD tasks among different works are MNIST, permuted MNIST, and PTB. Moreover, most works \cite{chang2018antisymmetricrnn, cooijmans2016recurrent, zhang2018stabilizing}, unlike us and \cite{kusupati2018nips}, evaluate on at most 2-3 datasets. We chose one dataset from each sequential decision tasks category. So for Language Modeling: PTB over the easier Text8; for Classification: Google-12/30 over IMDB; for activity recognition: DSA-19, HAR-2 etc. 

We include traditional RNN tasks: (a) star rating prediction on a scale of 1 to 5 of Yelp reviews  \cite{Yelp2017}, (b) classification of MNIST images on a pixel-by-pixel sequence  (or a fixed random permuted sequence) \cite{Lecun98gradient-basedlearning}, and (c) language modeling on the Penn Treebank (PTB) data set \cite{McAuley2013PTB}. We also include activity recognition tasks: (a) Google-30 \cite{warden2018google30} and Google-12, \ie detection of utterances of 30 and 10 commands plus background noise and silence and (b) HAR-2 \cite{Anguita2012HAR} and DSA-19 \cite{Altun2010DSA19}, \ie Human Activity Recognition (HAR) from an accelerometer and gyroscope on a Samsung Galaxy S3 smartphone and Daily and Sports Activity (DSA) detection from a resource-constrained IoT wearable device with 5 Xsens MTx sensors having accelerometers, gyroscopes and magnetometers on the torso and four limbs. 

All the data sets are publicly available and their pre-processing and feature extraction details are provided in \cite{kusupati2018nips}. Our experiments use the same test/train split as \cite{kusupati2018nips} for all the data sets. We further utilized 20\% of training data for validation to tune hyperparameters, and then the algorithms were trained on the full training set and the results were reported on the publicly available test set. Table \ref{table:1} lists the statistics of all the data sets. 

{\bf Noise Padded Datasets}: Additionally, as in \cite{chang2018antisymmetricrnn}, we \textit{induce} LTD by padding CIFAR-$10$ with noise exactly replicating their setup. We do so on MNIST (above table). \textit{Intuition} (Ideally): Say, noise is padded at $t>\tau$ and resulting $Wx_t$ is close to mean value. For ERNN the resulting states ceases to be updated. So ERNN recalls last informative state $h_\tau$ (modulo const) unlike RNNs/variants!

\textbf{Baseline Algorithms and Implementation:} For all our experiments, we used the parametrized update formulation in Eq. \ref{eq.experiments} for ERNN. $K$ in the table denotes the number of fixed point recursion steps (higher $K$ approximates the equilibrium better).


\begin{wraptable}{r}{.58\linewidth}
 	\vspace{-25pt}
	\begin{minipage}[b]{\linewidth}
		\begin{center}\scriptsize 
            \setlength{\tabcolsep}{.5pt}
            \caption{\footnotesize Comparison on benchmark data sets: $K$ denotes pre-defined recursions embedded in graph to reach equillibrium.}
            \begin{tabular}{|ccccccc|} 
             \hline
             Data set & Algorithm & \begin{tabular}[c]{@{}c@{}}Accuracy\\ (\%)\end{tabular} & \begin{tabular}[c]{@{}c@{}}Model\\ Size (KB)\end{tabular} & \begin{tabular}[c]{@{}c@{}}Train\\ Time (hr)\end{tabular} & \begin{tabular}[c]{@{}c@{}}Test\\ Time (ms)\end{tabular} & \begin{tabular}[c]{@{}c@{}}\#Params\end{tabular}  \\ [0.5ex] 
             \hline\hline
             HAR-2 & FastRNN & 94.50 & 29 & 0.063 &  \textbf{0.01} & 7.5k \\ 
              & FastGRNN-LSQ & 95.38 & 29 & 0.081 &  0.03 & 7.5k \\ 
              & RNN & 91.31 & 29 & 0.114 & \textbf{0.01} & 7.5k  \\ 
              & SpectralRNN & 95.48 & 525 & 0.730 & 0.04 & 134k \\ 
              & LSTM & 93.65 & 74 & 0.183 &  0.04 & 16k\\ 
              & GRU & 93.62 & 71 & 0.130 & 0.02 & 16k \\ 
              & Antisymmetric & 93.15 & 29 & 0.087 & \textbf{0.01} & 7.5k \\ 
              & {\bf ERNN(K=1)} & 95.32 & \textbf{17} & 0.061 & \textbf{0.01} & \textbf{4k} \\
              & {\bf ERNN(K=3)} & 95.52 & \textbf{17} & 0.081 & 0.02 & \textbf{4k} \\
              & {\bf ERNN(K=5)} & \textbf{96.30} & 18 & \textbf{0.018} & 0.03 & \textbf{4k} \\
              \hline
             DSA-19 & FastRNN & 84.14 & 97 & 0.032 & \textbf{0.01} &  17.5k \\ 
              & FastGRNN-LSQ & 85.00 & 208 & 0.036 & 0.03 &  35k \\ 
              & RNN & 71.68 & 20 & 0.019 & \textbf{0.01} & \textbf{3.5k}  \\ 
              & SpectralRNN & 80.37 & 50 & 0.038 & 0.02  & 8.8k \\ 
              & LSTM & 84.84 & 526 & 0.043 & 0.06 & 92k  \\
              & GRU & 84.84 & 270 & 0.039 & 0.03 &  47k \\ 
              & Antisymmetric & 85.37 & 32 & 0.031 & \textbf{0.01} &  8.3k\\ 
              & {\bf ERNN(K=1)} & \textbf{88.11} & \textbf{19} & 0.015 & \textbf{0.01} & \textbf{3.5k} \\
              & {\bf ERNN(K=3)} & 85.20 & \textbf{19} & 0.020 & 0.02 &  \textbf{3.5k} \\
              & {\bf ERNN(K=5)} & 87.37 & 20 & \textbf{0.005} &  0.03 & \textbf{3.5k} \\
              \hline
             Google-12 & FastRNN & 92.21 & 56 & 0.61 & \textbf{0.01} & 12k   \\ 
              & FastGRNN-LSQ & 93.18 & 57 & 0.63 & 0.03 & 12k  \\ 
              & RNN & 73.25 & 56 & 1.11 & \textbf{0.01}  & 12k  \\ 
              & SpectralRNN & 91.59 & 228 & 19.0 & 0.05 &  49k \\ 
              & LSTM & 92.30 & 212 & 1.36 & 0.05 &  45k \\ 
              & GRU & 93.15 & 248 & 1.23 & 0.05 &  53k \\ 
              & Antisymmetric & 89.91 & 57 & 0.71 & \textbf{0.01} & 12k \\ 
              & {\bf ERNN(K=1)} & 93.93 & \textbf{36} & 0.20 & \textbf{0.01} &  \textbf{8.1k} \\
              & {\bf ERNN(K=3)} & 94.16 & 37 & 0.33 & 0.03 &  \textbf{8.1k} \\
              & {\bf ERNN(K=5)} & \textbf{94.71} & 38 & \textbf{0.17} & 0.05 & \textbf{8.1k}  \\
              \hline
             Google-30 & FastRNN & 91.60 & 96 & 1.30 & \textbf{0.01}  & 18k  \\ 
              & FastGRNN-LSQ & 92.03 & 45 & 1.41 & \textbf{0.01}  & \textbf{8.5k}  \\ 
              & RNN & 80.05 & 63 & 2.13 &  \textbf{0.01} &  12k \\ 
              & SpectralRNN & 88.73 & 128 & 11.0 & 0.03 &  24k  \\ 
              & LSTM & 90.31 & 219 & 2.63 & 0.05  &  41k \\ 
              & GRU & 91.41 & 257 & 2.70 & 0.05  &  48.5k \\ 
              & Antisymmetric & 90.91 & 64 & 0.54 & \textbf{0.01}  & 12k \\ 
              & {\bf ERNN(K=1)} & 93.77 & \textbf{44} & 0.44 & \textbf{0.01}  & \textbf{8.5k}  \\
              & {\bf ERNN(K=3)} & 91.30 & \textbf{44} & 0.44 & 0.03  & \textbf{8.5k}  \\
              & {\bf ERNN(K=5)} & \textbf{94.23} & 45 & \textbf{0.44} &  0.05  & \textbf{8.5k} \\
              \hline
             Pixel-MNIST & FastRNN & 96.44 & 166 & 15.10 & \textbf{0.05}  &  33k \\  
              & FastGRNN-LSQ & \textbf{98.72} & 71 & 12.57 &  0.06 &  14k \\ 
              & RNN & 94.10 & 71 & 45.56 & \textbf{0.05}  &   14k \\ 
              & LSTM & 97.81 & 265 & 26.57 & 0.1  & 53k  \\ 
              & GRU & 98.70 & 123 & 23.67 & 0.07  & 25k    \\ 
              & Antisymmetric & 98.01 & 70 & 8.61 & \textbf{0.05} & 14k \\ 
              & {\bf ERNN(K=1)} & 97.73 & \textbf{20} & \textbf{2.83} &  \textbf{0.05} &  \textbf{4k}   \\
              & {\bf ERNN(K=2)} & 98.13  & 23 &  3.11 & 0.06  &    \textbf{4k} \\
              & {\bf ERNN(K=3)} & 98.13 & 25 & 2.93 & 0.07 & \textbf{4k}    \\
              \hline
             Yelp-5 & FastRNN & 55.38 & 130 & 3.61 & \textbf{0.4}  &  32.5k \\  
              & FastGRNN-LSQ & {\bf 59.51} & 130 & 3.91 & 0.7  &  32.5k \\ 
              & RNN & 47.59 & 130 & 3.33 & \textbf{0.4}  & 32.5k \\ 
              & SpectralRNN & 56.56 & \textbf{89} & 4.92 &  0.3 &  \textbf{22k} \\ 
              & LSTM & 59.49 & 516 & 8.61 &  1.2 & 129k  \\ 
              & GRU & 59.02 & 388 & 8.12 & 0.8  &  97k \\ 
              & Antisymmetric & 54.14 & 130 & 2.61 & \textbf{0.4} & 32.5k \\ 
              & {\bf ERNN(K=1)} & 58.16 & 97.67 & \textbf{0.31} & \textbf{0.4}  & 25k \\
              & {\bf ERNN(K=2)} & 59.01 & 98.84 & \textbf{0.31} & 0.7  & 25k  \\
              & {\bf ERNN(K=3)} & 59.34  & 100 & 1.16 & 1.0  &  25k \\
             \hline
             Permute-MNIST & FastRNN & 92.68 & 44 & 9.32 & \textbf{0.03} & 8.75k \\ 
             & LSTM & 92.61 & 176 & 19.31 & 0.12 & 35k \\ 
             & Antisymmetric & 93.59 & 70 & 4.75 & \textbf{0.03} & 14k \\ 
             & ERNN(K=1) &  \textbf{95.62} &  \textbf{41} &  \textbf{2.41} & 0.04  &  \textbf{8k} \\ 
             \hline
             Noisy-MNIST & FastRNN & 98.12 & 58 & 8.93 &  \textbf{0.09} & 11k \\ 
             & LSTM & 10.31 & 232 & 19.43 & 0.36 & 44k \\ 
             & Antisymmetric & 97.76 & 54 & 5.21 &  \textbf{0.09} & 10k \\ 
             & ERNN(K=1) &  \textbf{98.48} &  \textbf{32} &  \textbf{2.39} &  \textbf{0.09} &  \textbf{6k} \\ 
             \hline
             Noisy-CIFAR & FastRNN & 45.76 & 81 & 11.61 & 0.19 & 16k \\ 
             & LSTM & 11.60 & 324 & 23.47 & 0.76 & 64k \\ 
             & Antisymmetric & 48.63 & 81 & 5.81 & 0.23 & 16k \\ 
             & ERNN(K=1) &  \textbf{54.50} &  \textbf{58} &  \textbf{2.47} &  \textbf{0.14} &  \textbf{11.5k} \\ 
             \hline
            \end{tabular}
            \label{table:2}
            \vspace{-5mm}
		\end{center}
	\end{minipage}
	\vspace{-10pt}
\end{wraptable}


We compared ERNN with standard RNN, SpectralRNN \cite{2018SpectralRNN},  LSTM \cite{hochreiter1997long}, GRU \cite{cho2014properties}, AntisymmetricRNN \cite{chang2018antisymmetricrnn}, FastRNN and FastGRNN-LSQ (\ie FastGRNN without model compression but achieving better accuracy and lower training time) \cite{kusupati2018nips}. 

Since reducing model size is not our concern, we did not pursue model compression experiments and thus did not compare ERNN with FastGRNN directly, though potentially all the compression techniques in FastGRNN could be applicable to ERNN as well. Since AntisymmetricRNN \cite{chang2018antisymmetricrnn} has similar update rule as FastRNN, we do not report its performance in the table (see also Fig.~\ref{fig:convergence_rate_baseline} showing similar behaviour as FastRNN). 


We used the publicly available implementation \cite{EdgeML2018} for FastRNN and FastGRNN-LSQ to validate results. For others, we cited the corresponding numbers for the other competitors from \cite{kusupati2018nips}. All the experiments were run on an Nvidia GTX 1080 GPU with CUDA 9 and cuDNN 7.0 on a machine with Intel Xeon 2.60 GHz GPU with 20 cores. We found that FastRNN and FastGRNN-LSQ can be trained to perform similar accuracy as reported in \cite{kusupati2018nips} using slightly longer training time on our machine. This indicates that potentially all the other competitors can achieve similar accuracy using longer training time.

\textbf{Hyper-parameters:} We used grid search and fine-grained validation wherever possible to set the hyper-parameters of each algorithm, or according to the settings published in \cite{kusupati2018nips} otherwise (\eg number of hidden states). Both the learning rate and $\eta$'s were initialized to $10^{-2}$. Due to fast covergence in ERNN, we halved the learning rate periodically, where the period was learned based on the validation set. Replicating this on FastRNN or FastGRNN-LSQ does not achieve the maximum accuracy reported in the paper. The batch size of $128$ seems to work well across all the data sets. We used ReLU as the non-linearity and Adam \cite{2015Adam} as the optimizer for all the experiments. We examined the training behavior of ERNN using different $K$'s as illustrated in Fig. \ref{fig:norm_grad}. Since we empirically observed that ERNN achieves equilibrium in less than $5$ steps, we show results for $K=1,3,5$ for all data sets, except large-scale data sets (Pixel-MNIST, Yelp-5) where we are bottle-necked by the GPU memory. 

\textbf{Evaluation Criteria:} The primary focus in this paper is on achieving on-par or better results than state-of-the-art RNNs with much better convergence rate in unstable scenarios, similar to FastRNN. Therefore, following the exact experimental settings in FastRNN and FastGRNN-LSQ, we reported our model sizes (excluding the word-vector embedding storage for PTB and Yelp data set), training time and accuracy (perplexity on the PTB data set), and quoted the numbers from \cite{kusupati2018nips}. We are aware that better numbers on some data sets may exist such as \cite{gong2018frage, cooijmans2016recurrent} with different settings where training instability, is not an issue. Since increasing $K$ to achieve the equilibrium point has inherent computational cost, we also report the prediction time taken for one example as an ablative study.

\begin{wraptable}{r}{.5\linewidth}
	\vspace{-20pt}
	\begin{minipage}[b]{\linewidth}
	    \caption{\footnotesize PTB Language Modeling: 1 Layer. 
	    To be consistent with our other experiments we used a low-dim $\mathbf{U}$; For this size our results did not significantly improve with $K$. This is the dataset of \cite{kusupati2018nips} which uses sequence length $300$ as opposed to 30 in the conventional PTB. Test time, \#Params are tabulated in supplementary}
		\begin{center}\footnotesize
            \setlength{\tabcolsep}{.5pt}
            \begin{tabular}{|c c c c|} 
                \hline
                Algorithm & \begin{tabular}[c]{@{}c@{}}Test Perplexity\end{tabular} & \begin{tabular}[c]{@{}c@{}}Model\\ Size (KB)\end{tabular} & \begin{tabular}[c]{@{}c@{}}Train\\ Time (min)\end{tabular} \\ [0.5ex] 
                \hline\hline
                FastRNN & 127.76 & 513 & 11.20 \\  
                FastGRNN-LSQ & 115.92 & 513 & 12.53 \\ 
                RNN & 144.71 & {\bf 129} & 9.11 \\ 
                SpectralRNN & 130.20 & 242 & - \\ 
                LSTM & 117.41 & 2052 & 13.52 \\ 
                UGRNN & 119.71 & 256 & 11.12 \\ 
                {\bf ERNN(K=1)} & \textbf{115.71} & 288 & {\bf 7.11} \\
                \hline
            \end{tabular}
            \label{table:3}
		\end{center}
	\end{minipage}
\vspace{-13pt}
\end{wraptable}




\textbf{Results:}
Table \ref{table:2} and Table \ref{table:3} compare the performance of ERNN to state-of-the-art RNNs. Four points are worth noticing about ERNN's performance. First, ERNN's prediction gains over a standard RNN ranged from $4.03\%$ on Pixel-MNIST data set to $21.46\%$ on Google-12 data set. Similar observations follow for other RNN variants, demonstrating the superiority of our approach. Second, ERNN's prediction accuracy always surpassed FastRNN's (euler discretized ODE rule) prediction accuracy. This indeed shows the advantage of operating the ODE RNNs on equilibrium manifold. This can also be observed in Fig. \ref{fig:convergence_rate_baseline} where ERNNs converge much faster than AntisymmetricRNN (ODE similar to FastRNN). Third, ERNN surpasses gating based FastGRNN-LSQ on $4$ out of $6$ data set in terms of prediction accuracy with $3.11\%$ on DSA-19 data set and $2.23\%$ on Google-30 data set. Fourth, and most importantly, ERNN's training speedups over FastRNN as well as FastGRNN-LSQ range from $3.5$x on HAR-2 data set to $5.2$x on Pixel-MNIST data set. This emphasizes the fact that operating RNNs on the equilibrium manifold leads to superior accuracy than gating based methods with substantially better training efficiency. Thus while fixed-point recursions lead to additional computation, we more than offset it with faster convergence. Also note that the model size of ERNN is always smaller than the model size of either FastRNN or FastGRNN-LSQ.

\begin{wrapfigure}{r}{.5\linewidth}
	\vspace{-25pt}
	\begin{center}
		\includegraphics[width=\linewidth]{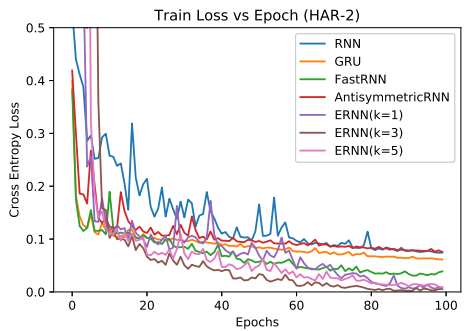}
		\vspace{-5mm}
		\caption{\footnotesize Comparison between baseline RNNs and ERNN on training convergence on HAR-2 dataset.}
		\label{fig:convergence_rate_baseline}
	\end{center}
	\vspace{-15pt}
\end{wrapfigure}

\textbf{Convergence Comparison wrt Baseline:} To further verify the advantage of our approach in terms of rate of convergence, we show the training behavior of different approaches in Fig. \ref{fig:convergence_rate_baseline}. As we see, ERNNs converge significantly faster than the baselines while achieving lower losses. Furthermore, ERNN($K=3$) curve stays below ERNN($K=1$), indicating that finer approximate solutions for equilibrium points lead to faster convergence as well as better generalization (see Table \ref{table:2}). The training time reported in Table \ref{table:2} refers to a convergent model with best accuracy.

\textbf{Convergence Comparison wrt $K$:} To demonstrate the change in rate of convergence in ERNN with higher values of $K$, we plotted training loss against epochs for each $K=1,3,5$ for HAR-2 dataset. Fig. \ref{fig:convergence_rate_baseline} shows higher values decrease training loss faster. From table \ref{table:2}, it can also be verified that this faster convergence rate leads to better generalization. Interestingly the training time required for $K=5$ is much smaller than $K=1$, implying that better approximation to equilibrium leads to faster convergence and better generalization offsetting extra computational cost.

%% file: supplementary.tex
\section{Local Convergence with Linear Rate}
Recall that we rewrite the fixed-point constraints in our ERNNs as the following ODE:
\begin{align}\label{eqn:ode}
\mathbf{h}'_k(t) = F(\mathbf{h}_k(t)) \stackrel{def}{=} f(\mathbf{h}_k(t) + \mathbf{h}_{k-1}, \mathbf{x}_k; \alpha) - \gamma(\mathbf{h}_k(t) + \mathbf{h}_{k-1}); \,\,\, \h_k(0)=\mathbf{0}.
\end{align}
Then based on the Euler method, we have the following update rule for solving fixed-points:
\begin{align}\label{eqn:update_rule}
    \hspace{-2mm}\mathbf{h}_{k}^{(i+1)}  = &\mathbf{h}_{k}^{(i)} + \eta_k^{(i)}F(\mathbf{h}_k^{(i)})\\ &= \mathbf{h}_{k}^{(i)} + \eta_k^{(i)}[f(\mathbf{h}_k^{(i)} + \mathbf{h}_{k-1}, \mathbf{x}_k; \alpha) - (\mathbf{h}_k^{(i)} + \mathbf{h}_{k-1})].
\end{align}

{\em Inexact Newton methods} \cite{dembo1982inexact} refer to a family of algorithms that aim to solve the equation system $F(\mathbf{z})=\mathbf{0}$ approximately at each iteration using the following rule:
\begin{align}\label{eqn:inexact}
    \mathbf{z}^{(i+1)} = \mathbf{z}^{(i)} + \mathbf{s}^{(i)}, \, \mathbf{r}^{(i)} = F(\mathbf{z}^{(i)}) + \nabla F(\mathbf{z}^{(i)})\mathbf{s}^{(i)},     
\end{align}
where $\nabla F$ denotes the (sub)gradient of function $F$, and $\mathbf{r}^{(i)}$ denotes the error at the $i$-th iteration between $F(\mathbf{z}^{(i)})$ and $\mathbf{0}$.

By drawing the connection between Eq. \ref{eqn:update_rule} and Eq. \ref{eqn:inexact}, we can set $\mathbf{z}^{(i)}\equiv\mathbf{h}_{k}^{(i)}$ and  $\mathbf{s}^{(i)}\equiv\eta_{k}^{(i)}F(\mathbf{h}_{k}^{(i)})$. Then based on Eq. \ref{eqn:inexact} we have 
\begin{align}\label{eqn:r}
\mathbf{r}^{(i)} = F(\mathbf{h}_{k}^{(i)}) + \eta_{k}^{(i)}\nabla F(\mathbf{h}_{k}^{(i)})F(\mathbf{h}_{k}^{(i)}).
\end{align}

\begin{lemma}[Thm. 2.3 in \cite{dembo1982inexact}]\label{lem:1}
Assume that 
\begin{align}\label{eqn:tau}
    \frac{\|\mathbf{r}^{(i)}\|}{\|F(\mathbf{z}^{(i)})\|}\leq\tau<1, \forall k,
\end{align}
where $\|\cdot\|$ denotes an arbitrary norm and the induced operator norm. There exists $\varepsilon>0$ such that, if $\|\mathbf{z}^{(0)}-\mathbf{z}^*\|\leq\varepsilon$, then the sequence of inexact Newton iterates $\{\mathbf{z}^{(i)}\}$ converges to $\mathbf{z}^*$. Moreover, the convergence is linear in the sense that $\|\mathbf{z}^{(i+1)}-\mathbf{z}^*\|_*\leq \tau\|\mathbf{z}^{(i)}-\mathbf{z}^*\|_*$,
where $\|\mathbf{y}\|_*=\|\nabla F(\mathbf{z}^*)\mathbf{y}\|$.
\end{lemma}

\begin{thm}[Local Convergence with Linear Rate]\label{thm:convergence}
Assume that the function $F$ in Eq. \ref{eqn:ode} and the parameter $\eta_{k}^{(i)}$ in Eq. \ref{eqn:update_rule} satisfy 
\begin{align}\label{eqn:euler}
    [\eta_{k}^{(i)}]^2 \|\nabla F(\mathbf{h}_{k}^{(i)})F(\mathbf{h}_{k}^{(i)})\|^2 + 2\eta_{k}^{(i)}F(\mathbf{h}_{k}^{(i)})^T\nabla F(\mathbf{h}_{k}^{(i)})F(\mathbf{h}_{k}^{(i)}) < 0, \forall i, \forall k.
\end{align}
Then there exists $\epsilon>0$ such that if $\|\mathbf{h}_{k}^{(0)} - \mathbf{h}_{k}\|\leq\epsilon$ where $\mathbf{h}_{k}$ denotes the fixed point, the sequence $\{\mathbf{h}_k^{(i)}\}$ generated by the Euler method converges to the equilibrium solution in ${\cal M}(\mathbf{h}_{k-1},\mathbf{x}_k)$ locally with linear rate.
\end{thm}
\begin{proof}
By substituting Eq. \ref{eqn:r} into Eq. \ref{eqn:tau}, to prove local convergence we need to guarantee 
\begin{align}\label{eqn:6}
    \|F(\mathbf{h}_{k}^{(i)}) + \eta_{k}^{(i)}\nabla F(\mathbf{h}_{k}^{(i)})F(\mathbf{h}_{k}^{(i)})\| < \|F(\mathbf{h}_{k}^{(i)})\|.
\end{align}
By taking the square of both sides in Eq. \ref{eqn:6}, we can show that Eq. \ref{eqn:6} is equivalent to Eq. \ref{eqn:euler}. We then complete our proof.

\end{proof}

\begin{cor}
Assume that $\|\mathbf{I}+\eta_{k}^{(i)}\nabla F(\mathbf{h}_{k}^{(i)})\|<1, \forall i, \forall k$ holds. Then the forward propagation using Eq. \ref{eqn:update_rule} is stable and our sequence $\{\mathbf{h}_{k}^{(i)}\}$ converges locally with linear rate.
\end{cor}
\begin{proof}
By substituting Eq. \ref{eqn:r} into Eq. \ref{eqn:tau} and based on the assumption in the corollary, we have
\begin{align}
    \frac{\|\mathbf{r}^{(i)}\|}{\|F(\mathbf{h}_{k}^{(i)})\|} & = \frac{\|F(\mathbf{h}_{k}^{(i)}) + \eta_{k}^{(i)}\nabla F(\mathbf{h}_{k}^{(i)})F(\mathbf{h}_{k}^{(i)})\|}{\|F(\mathbf{h}_{k}^{(i)})\|} \nonumber \\ & \leq\frac{\|\mathbf{I}+\eta_{k}^{(i)}\nabla F(\mathbf{h}_{k}^{(i)})\|\|F(\mathbf{h}_{k}^{(i)})\|}{\|F(\mathbf{h}_{k}^{(i)})\|} < 1.
\end{align}
Further based on Prop. 2 in \cite{chang2018antisymmetricrnn} and Thm. \ref{thm:convergence}, we then complete our proof.
\end{proof}

\section{Theoretical Verification}
In this section, we include some experiments to show that our theoretical assumptions hold true.

\begin{figure}[h!]
  \includegraphics[width=\linewidth]{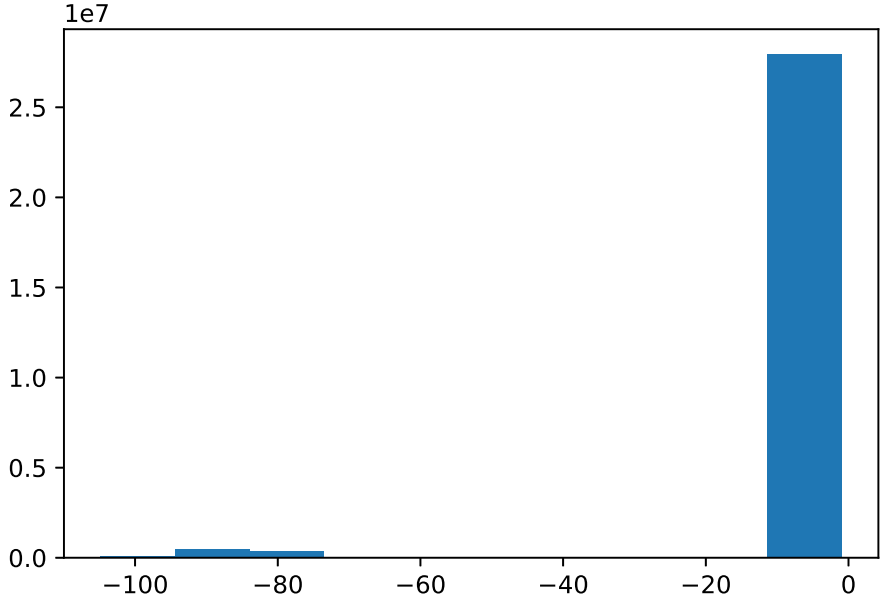}
  \caption{\footnotesize Histogram of the eigenvalues of $\nabla \phi \mathbf{U} - \mathbf{I}$ for ERNN on HAR-2 dataset.}
 		\label{fig:eigen_values_histogram}
\end{figure}

\subsection{Non-Singularity of the matrix \textbf{D}} For our ERNN parametrization to satisfy the conditions of having equillibrium points to be locally asymptotically stable, the eigen values of the matrix $D=(\nabla \phi(\cdot) U - \gamma I)$ should be negative. We plot a histogram of the eigenvalues of $D$ for all the points in the HAR-2 dataset. As illustrated in the figure \ref{fig:eigen_values_histogram}, all the eigenvalues are negative.


\subsection{Different Activation Function} We also performed some experiments for sigmoid activation on HAR-2 dataset. The results for this variant also follow similar pattern as we saw in ReLU variant.

\begin{table}[h!]
\centering
\begin{tabular}{|ccccccc|} 
 \hline
 Data set & Algorithm & \begin{tabular}[c]{@{}c@{}}Accuracy\\ (\%)\end{tabular} & \begin{tabular}[c]{@{}c@{}}Model\\ Size (KB)\end{tabular} & \begin{tabular}[c]{@{}c@{}}Train\\ Time (hr)\end{tabular} & \begin{tabular}[c]{@{}c@{}}Activation\end{tabular} & \begin{tabular}[c]{@{}c@{}}\#Params\end{tabular}  \\ [0.5ex] 
 \hline\hline
 HAR-2 & ERNN(K=1) & 95.32 & \textbf{17} & 0.061 & ReLU & \textbf{4k} \\
      & ERNN(K=3) & 95.52 & \textbf{17} & 0.081 & ReLU & \textbf{4k} \\
      & {\bf ERNN(K=5)} & \textbf{96.30} & 18 & \textbf{0.018} & ReLU & \textbf{4k} \\
      & ERNN(K=1) & 92.16 & \textbf{17} & 0.065 & Sigmoid & \textbf{4k} \\
      & ERNN(K=3) & 93.35 & \textbf{17} & 0.078 & Sigmoid & \textbf{4k} \\
      & {\bf ERNN(K=5)} & 95.30 & 18 & 0.020 & Sigmoid & \textbf{4k} \\
  \hline
\end{tabular}
\vspace{3mm}
\caption{HAR-2 dataset (Sigmoid, ReLU activations): $K$ denotes pre-defined recursions embedded in graph to reach equillibrium.}
\label{table:diff_activation}
\end{table}


\begin{wrapfigure}{r}{.4\linewidth}
	\vspace{-40pt}
	\begin{center}
		\includegraphics[width=\linewidth]{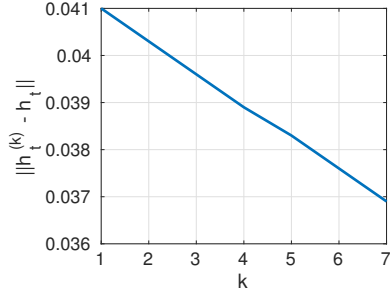}
		\vspace{-7mm}
		\caption{\footnotesize Linear convergence in ERNN.}
		\label{fig:convergence}
	\end{center}
	\vspace{-30pt}
\end{wrapfigure}

\subsection{Linear Rate of Convergence to Fixed Point}
Empirically we verify the local convergence to a fixed point with linear rate by comparing the Euclidean distance between the approximate solutions, $\mathbf{h}_t^{(k)}$, using Eq.~\ref{eqn:update_rule} and the fixed points, $\mathbf{h}_t$, computed using {\sc fsolve} from {\sc scipy}. The learnable parameters are initialized suitably and then fixed. We illustrate our results in Fig. \ref{fig:convergence}, which clearly demonstrates that the approximate solutions tend to converge with linear rate.

\begin{align}\label{eqn:update_rule}
    \hspace{-2mm}\mathbf{h}_{k}^{(i+1)} & = \mathbf{h}_{k}^{(i)} + \eta_k^{(i)}F(\mathbf{h}_k^{(i)}) = \mathbf{h}_{k}^{(i)} + \eta_k^{(i)}[f(\mathbf{h}_k^{(i)} + \mathbf{h}_{k-1}, \mathbf{x}_k; \alpha) - (\mathbf{h}_k^{(i)} + \mathbf{h}_{k-1})],
\end{align}

\subsection{PTB Language Modelling}
Table \ref{table:2} reports all the evaluation metrics for the PTB Language modelling task with $1$ layer as setup by \cite{kusupati2018nips}, including test time and number of parameters (which we omitted from the main paper due to lack of space).

\begin{table}[h!]
\centering

\begin{tabular}{|c c c c c c|} 
                \hline
                Algorithm & \begin{tabular}[c]{@{}c@{}}Test Perplexity\end{tabular} & \begin{tabular}[c]{@{}c@{}}Model\\ Size (KB)\end{tabular} & \begin{tabular}[c]{@{}c@{}}Train\\ Time (min)\end{tabular} & \begin{tabular}[c]{@{}c@{}}Test\\ Time (ms)\end{tabular} & \begin{tabular}[c]{@{}c@{}}\#Params\end{tabular}\\ [0.5ex] 
                \hline\hline
                FastRNN & 127.76 & 513 & 11.20 & 1.2 & 52.5k \\  
                FastGRNN-LSQ & 115.92 & 513 & 12.53  & 1.5 & 52.5k\\ 
                RNN & 144.71 & {\bf 129} & 9.11  & {\bf 0.3} & {\bf 13.2k} \\ 
                SpectralRNN & 130.20 & 242 & -  & 0.6 & 24.8k \\ 
                LSTM & 117.41 & 2052 & 13.52  & 4.8 & 210k \\ 
                UGRNN & 119.71 & 256 & 11.12  & 0.6 & 26.3k \\ 
                {\bf ERNN(K=1)} & \textbf{115.71} & 288 & {\bf 7.11}  & 0.6 & 29.5k \\
                \hline
            \end{tabular}

\vspace{3mm}
\caption{PTB Language Modeling: 1 Layer. 
	    To be consistent with our other experiments we used a low-dim $\mathbf{U}$; For this size our results did not significantly improve with $K$. This is the dataset of \cite{kusupati2018nips} which uses sequence length $300$ as opposed to 30 in the conventional PTB. }
\label{table:ptb_full_table}
\end{table}